%% file: paper.tex
\documentclass[twoside]{article}

\usepackage[accepted]{aistats2025}
\usepackage[utf8]{inputenc} 
\usepackage[T1]{fontenc}    
\usepackage{natbib}
\usepackage{hyperref}
\hypersetup{
    colorlinks=true,
    linkcolor=red,
    filecolor=magenta,      
    urlcolor=blue,
    citecolor=blue,
    pdfpagemode=FullScreen,
}
\usepackage{url}            
\usepackage{booktabs}       
\usepackage{amsfonts}       
\usepackage{nicefrac}       
\usepackage{microtype}
\usepackage{microtype}      
\usepackage{xcolor}         

\usepackage{amsmath,amssymb,amsthm}
\usepackage{mathtools,thmtools,stmaryrd,esint}
\usepackage{mathrsfs}
\declaretheorem[name=Theorem,within=section]{thm}
\declaretheorem[name=Lemma,numberlike=thm]{lemma}
\declaretheorem[name=Proposition,numberlike=thm]{prop}

\declaretheorem[name=Definition,numberlike=thm]{defn}

\declaretheorem[name=Example,numberlike=thm,style=remark]{example}

\begin{document}

%

%
\runningauthor{Ashwin Samudre, Mircea Petrache, Brian D. Nord, Shubhendu Trivedi}

\twocolumn[

\aistatstitle{Symmetry-Based Structured Matrices for Efficient Approximately Equivariant Networks}

\runningtitle{Efficient Approximately Equivariant Networks}



\aistatsauthor{
\textbf{Ashwin Samudre} \\
SFU
\And
\textbf{Mircea Petrache} \\
PUC Chile
\And
\textbf{Brian D. Nord} \\
FNAL \& UChicago
\And
\textbf{Shubhendu Trivedi} \\
Independent
}
\vspace{0.8cm}
]

\begin{abstract}
There has been much recent interest in designing neural networks (NNs) with relaxed equivariance, which interpolate between exact equivariance and full flexibility for consistent performance gains. In a separate line of work, structured parameter matrices with low displacement rank (LDR)---which permit fast function and gradient evaluation---have been used to create compact NNs, though primarily benefiting classical convolutional neural networks (CNNs). In this work, we propose a framework based on symmetry-based structured matrices to build approximately equivariant NNs with fewer parameters. Our approach unifies the aforementioned areas using Group Matrices (GMs), a forgotten precursor to the modern notion of regular representations of finite groups. GMs allow the design of structured matrices similar to LDR matrices, which can generalize all the elementary operations of a CNN from cyclic groups to arbitrary finite groups. We show GMs can also generalize classical LDR theory to general discrete groups, enabling a natural formalism for approximate equivariance. We test GM-based architectures on various tasks with relaxed symmetry and find that our framework performs competitively with approximately equivariant NNs and other structured matrix-based methods, often with one to two orders of magnitude fewer parameters.

\end{abstract}

\section{Introduction}\label{sec:intro}
\vspace{-1mm}
Over the past few years, the incorporation of symmetry in neural networks through \emph{group equivariance}~\citep{Cohen2016GroupEC} has started to mature as a fruitful line of research~\citep{Cohen2016GroupEC, batatia2024general, Bekkers20, cohenGeneralTheoryEquivariant2019, Cohen2019GaugeEC, Finzi2021APM, Kondor2018OnTG, levin2024any, Maron2019InvariantAE, mironenco2023lie, pearce2023brauer, RavanbakhshSP17,Weiler2021CoordinateIC,XuLDD22}. Such networks involve implementing a generalized form of convolution over groups---since linear equivariant maps over general fields are necessarily convolutional in nature~\citep{Bekkers20, cohenGeneralTheoryEquivariant2019, Kondor2018OnTG}, but can also be designed using ideas from invariant theory~\citep{BlumSmith2022MachineLA}, or by using non-linear equivariant maps via attention mechanisms~\citep{fuchs2020se, romero2020attentive, hutchinson2021lietransformer}. Successful applications have ranged over several areas, involving multiple data types~\citep{Townshend2021GeometricDL,Baek2021AccuratePO, Satorras2021EnEN, andersonCormorantCovariantMolecular2019, Klicpera2020DirectionalMP, Winkels2019PulmonaryND, kondor2018clebsch, Gordon2020PermutationEM, Sosnovik2021ScaleEI, Eismann2020HierarchicalRN, white-cotterell-2022-equivariant, ZhuWangGrasp2022, zaheer2017deep, kondor2018covariant, chen2021cyclically, gong2023general, maxson2024transferable, minartz2024equivariant, baker2023high}.

Despite the success of equivariant networks, several challenges remain. For instance, it has been demonstrated empirically~\citep{FinziSoftEquivariance, romero2022learning,ouderaa2022relaxing, wang2022approximately} that a strict equivariance constraint could harm performance---symmetry in real world tasks is rarely perfect and data measurements could be corrupted or systematically biased. Indeed, it is reasonable to expect that there could often be a mismatch between the symmetry encoded by the model and the symmetry present in the data. A natural modeling paradigm suggested by~\cite{FinziSoftEquivariance, romero2022learning,ouderaa2022relaxing, wang2022approximately} prescribes the design of more \emph{flexible} models that can calibrate the level of equivariance based on the data or task. This observation has been supported by a growing body of recent empirical~\citep{mcneela2023almost, wang2023relaxed, van2024learning, maile2022equivariance, urbano2023self} and theoretical work~\citep{Huang2023ApproximatelyEG, PetracheApproximate}.

In this paper, we study the principled design of approximately equivariant NNs with significantly reduced parameter counts and the attendant computational benefits. 
For this purpose, we examine a line of work independent of that on equivariant NNs, which is primarly concerned with designing compact deep learning models using structured matrix representations of fast transforms (see~\cite{rudra2023arithmetic} and the references therein).
One class of popular methods explicitly involves the design of computationally- and memory-effective structured matrices, intended to replace dense weight matrices in NNs~\citep{dao2022monarch,dao2019learning, dao2020kaleidoscope}. Yet another class of methods, proposed in the seminal works of~\cite{sindhwani2015structured, thomas2018learning}, instead operate with traditional structured matrices: Vandermonde (including polynomial and Chebyshev Vandermonde), Hankel, Toeplitz, Cauchy, Pick etc. Such matrices are a key object in engineering, especially in control theory, filtering theory, and signal processing~\citep{pan2012structured}. The special property of these matrices is that their compositions do not inherit their structure \emph{exactly}; however, they do so in an \emph{approximate sense}. This approximation is measured by the classical notion of the \emph{displacement rank} proposed by Kailath and co-workers~\citep{kailath1979displacement, kailath1995displacement, kailath1994generalized}---such matrices are said to possess \emph{displacement structure}, and a low degree of error on composite operations is indicated by a \emph{low displacement rank} (LDR). While~\cite{sindhwani2015structured, thomas2018learning} demonstrated that LDR matrices could be used to design efficient and compact NNs, they were only concerned with classical MLPs and CNNs and do not obviously extend to equivariant NNs. 


We start with the observation that general cyclic matrices---also known as circulant matrices, a sub-class of LDR matrices---are used to represent classical circular convolution\footnote{Or convolution over cyclic groups.}. It has already been observed by~\cite{thomas2018learning} that these readily yield a notion of approximate convolution (and approximate equivariance) in classical CNNs. It is this property, which is dependent on cyclic matrices being LDR, that~\cite{thomas2018learning} exploit to build low-resource CNNs. With motivation from the theory of regular representations, we use the forgotten notion of group matrices (GMs)\footnote{Group matrices are a precursor to the modern notion of regular representations, and pre-date modern group theory. For a history, we direct the reader to \cite{CHALKLEY1981121}.}, and obtain a new family of symmetry-based structured matrices that can be used to model convolution for general discrete groups. We first develop a generalization of classical CNNs using GMs by building analogues for each of their elementary operations. Then, simplifying older works on group matrices~\citep{gader1990displacement, waterhouse1992displacement}, we show that our formalism permits a principled notion of approximate equivariance for general discrete groups. In fact, somewhat analogous to~\cite{thomas2018learning}, we generalize LDR theory from cyclic matrices to their analogues for general discrete groups---thus giving a handle on quantifying error to exact equivariance. 
To align our exposition with that of modern equivariant NNs, we show how our framework naturally generalizes to the homogeneous spaces of discrete groups, as well as to compactly supported data on infinite discrete groups. On a variety of tasks, we show that our proposed method is able to consistently match or outperform baselines, both from approximately equivariant NNs and structured matrix-based frameworks, often with an order of magnitude or more, reduced parameter count.

This paper brings together the above-mentioned lines of work: equivariant NNs, approximately equivariant NNs, and structured matrix-based compression approaches. The result is a formalism that allows the construction of parameter-efficient approximately equivariant NNs with competitive performance. We summarize the main contributions of our paper below:
\vspace{-1mm}
\begin{itemize}
    \item We develop a formalism for constructing equivariant networks for general discrete groups using the notion of group matrices (GMs). We show how all the elementary operations in classical CNNs, such as taking strides and pooling, can be generalized to discrete groups using GMs. The resulting GM-CNNs involve the use of certain symmetry-based structured matrices which facilitate the construction of light-weight equivariant NNs. Further, the generalized pooling operation involves a form of group coarsening and is relevant to the development of group equivariant autoencoders.
    \item We present a simple procedure to construct GMs for larger discrete groups composed of direct products or semi-direct products of cyclic or permutation groups, for which GMs are easy to compute.
    \item We show that the GM-CNN formalism naturally permits a principled implementation of approximately equivariant group CNNs. Further, we connect GMs to the classical low displacement rank (LDR) theory, generalizing it from the specific case of LDR matrices with cyclic structure (thus cyclic groups) to discrete groups. We use the developed theory to quantify error from exact equivariance.
    \item We show our proposed framework extends to homogeneous spaces of discrete groups, and to compactly supported data on infinite discrete groups.
    \item On a variety of tasks, we show that our framework returns competitive performance, often with significantly reduced parameter count, compared to various approximately equivariant NNs and structured matrix-based frameworks. Our code is available at: \href{https://github.com/kiryteo/GM-CNN}{https://github.com/kiryteo/GM-CNN} 
\end{itemize}
\vspace{-5mm}
\section{Preliminaries and Formal Setup}
\label{sec:formal}
\vspace{-1mm}
\subsection{Matrices with Displacement Structure and Compressed Transforms in Deep Learning}
\label{sec:DSreview}

The work of \cite{thomas2018learning, sindhwani2015structured} considers traditional families of structured matrices---Hankel, Toeplitz, Vandermonde, Cauchy---to build compressed representations of MLPs and classical CNNs. In this section, we briefly describe the main idea. A matrix $M \in \mathbb{R}^{m \times n}$ could be called \emph{structured} if it can be represented in much fewer than $mn$ parameters. Examples include matrices where the elements have a simple formulaic relationship with other elements. The displacement operator approach, pioneered by Kailath \emph{et al.}~\citep{kailath1979displacement, kailath1995displacement}, consists of representing $M$ via matrices $A \in \mathbb{R}^{m \times m}$ and $B \in \mathbb{R}^{n \times n}$, which define a linear map known as the \textbf{displacement operator}, as follows:

\begin{defn}[Sylvester-type Displacement Operator:] 
Let $A$ and $B$ be fixed, not necessarily square matrices with compatible dimensions. The Sylvester-type displacement operator, denoted as $\nabla_{A,B}(M)$, is defined as the linear map: 
\looseness=-2
\vspace{-1mm}
\begin{equation*}
    \nabla_{A,B}(M): M \mapsto AM - MB,
\end{equation*}
\end{defn}
\vspace{-2mm}
\looseness=-2
where the difference $AM - MB$ is the \emph{residual} $R$. The rank of the matrix $R$ is called the \emph{displacement rank} (DR). The recovery of $M$ is immediate from $A, B, R$.

\textbf{Remark.} An alternative popular formulation of the displacement operator is the Stein-type displacement operator, which defines $\Delta_{A,B}(M): M \mapsto M - AMB$. For most purposes, the two formulations can be treated interchangeably~\citep{pan2012structured}.

The original formulation of DR was in the context of solutions of certain least square estimation problems of systems written in state-space form~\citep{kailath1979displacement, kailath1995displacement}. It was originally stated for matrices that were \emph{Toeplitz-like}, which were not \emph{exactly} shift-invariant, only \emph{approximately} so. The formulation in \cite{kailath1979displacement} showed that a Toeplitz-like matrix would be LDR. Later, such results were shown for all the classes of structured matrices listed above~\citep{pan2012structured, kailath1995displacement}. More importantly, many composite operations on LDR matrices---including transpose/inverse, addition, composition/multiplication, taking direct sums to construct a larger block matrix---still resulted in LDR matrices (see proposition 1 on closure in \cite{thomas2018learning}). It is this property that was used by \cite{sindhwani2015structured, thomas2018learning} to construct compressed representations of NNs. A particular form of the Toeplitz matrix, the circulant matrix, could be seen as implementing the usual circular convolution (see \ref{ex:circulant} below). Thus, a circulant-like matrix\footnote{A circulant-like matrix is circulant matrix to which an extremely sparse noise matrix is added. A circulant matrix implements circular convolution, but a matrix with some noise will implement approx. convolution. } would implement an approximate convolution (and thus approximate equivariance), and stacking them together in a NN would preserve the property of approximate equivariance. This point was noted and proved in Section 4 of~\cite{thomas2018learning} and was the basis for building approximately shift-equivariant CNNs in~\cite{thomas2018learning}. In the example below, we illustrate how working with circulant matrices corresponds to a convolution on the group $C_n$.

\begin{example}[Circulant Matrices]\label{ex:circulant}
    Note that for $A\in\mathbb R^{n\times n}$ equal to the canonical cyclic permutation $P$ on $n$ elements and $B=P^{-1}$, we have $\nabla_{P, P^{-1}}(M)=0$ iff $\Delta_{P,P^{-1}}(M)=0$, or iff $M$ is circulant. Then, we can interpret the entries of $v\in \mathbb R^n$ as encoding a function $f:C_n\to\mathbb R$ over the cyclic group $C_n=\mathbb Z/n\mathbb Z$. Thus, for circulant $M$ with first row $(m_{11},\dots, m_{1n})$ encoding map $\phi:C_n\to\mathbb R$, the multiplication $v\mapsto Mv$ corresponds to group-convolution operation on $C_n$.
\end{example}

The above exposition provides a tantalizing hint that building compressed approximately equivariant NNs in the spirit of ~\cite{thomas2018learning, sindhwani2015structured} could be a reasonable goal for general discrete groups---but we would need to identify special structured matrices for each group and ensure they obey an analogous LDR property as above. In the next section, we state group convolution and show how to express it in terms of special matrices called group matrices, which will prove key to achieving our goal.

\subsection{Group Matrices, Group Convolutions, and CNNs} 
\label{sec:gms}
In this section, we present our formulation of GM-CNNs. We first define group matrices, classical group convolution and reformulate group convolution in terms of GMs. Following which we show GMs naturally allow to generalize all the elementary operations of classical CNNs, including strides and pooling, to any discrete group. Additionally, we provide a procedure to construct GMs of certain larger discrete groups efficiently. Finally, we show that GMs enjoy a natural set of closure properties.

For ease of exposition, we describe the GM-CNN framework in this section using images as inputs, but it readily generalizes to general data types as will become clear later. 
In particular, we consider images of size $n\times n$, modulo the periodicity structure of each side, which is identified with the opposite side of the image. 
Such images can be considered as functions over copies of the group $G=C_n\times C_n$, which permits representing convolutions (as in classical CNNs\footnote{To be fully precise, in fact, classical CNNs use convolution with \emph{infinite}) group $\mathbb Z\times \mathbb Z$, which gives some extra technical difficulties as we need to keep track of image finite supports. This is solved in classical CNNs by padding. Later, in \S\ref{app:padding} we show how to extend our framework to general discrete infinite groups.}) as a group convolution on $G$. 

Since we seek a framework that uses symmetry-based structured matrices towards specific ends, we depart from the customary treatment of group convolutions that uses \emph{group representations}. 
We instead use the formalism of \textit{group matrices} (GMs) of discrete groups; while the GM formalism has a long history, it has mostly been forgotten and has not yet been used to represent general group convolutions. 
We first define the requisite notions of group matrices and group diagonals:

\textbf{Group matrices.} If $G$ is a group with cardinality $N$, then a matrix $M\in \mathbb R^{N\times N}$ 
is a \emph{group matrix} of $G$ if there exists a labeling of the rows and columns of $M$ by elements of $G$ such that whenever $gh=g'h'$ (with $g,h,g',h' \in G$) we have an identification of entries labeled by $(g,h), (g',h')$ i.e. that $M_{gh}=M_{g'h'}$.

\textbf{Group diagonals.} For $g\in G$ the \emph{group diagonal} matrix $B_g$ associated with $g$ is a particular group matrix with entries in $\{0,1\}$, whose entries are defined as
\begin{equation}\label{eq:gdiag}
    (B_g)_{h,h'}=\delta(h=gh'),
\end{equation}
where, $\delta(h=gh')$ equals $1$ if $h=gh'$, and $0$ otherwise (similar to the Kronecker delta notation). 
We observe that \textbf{group diagonals form a basis of the space of group matrices with entries in $\mathbb{X}=\mathbb R$}. 
$B_g$ are extremely sparse matrices with exactly one non-zero entry per row.
Thus, we can store them as arrays of size $|G|$, listing only the column indices of these entries, rather than as $|G|\times |G|$ matrices.


\textbf{Group convolution.} For standard group convolution, if $G$ is a finite group and $\phi,\psi:G\to \mathbb R$, then convolution of these functions is 
\begin{equation}\label{eq:gconv}
    \phi\star \psi(x):=\sum_{g\in G}\phi(g)\psi(g^{-1}x)=\sum_{\substack{g,h\in G:\\ gh=x}}\phi(g)\psi(h).
\end{equation}
In the second expression, we write the set of pairs $(g,g^{-1}x), g\in G$ implicitly as pairs $(g,h)\in G\times G$ satisfying the condition $hg=x$.
Observe that GMs permit a simple formulation for group convolutions for finite groups because they encode the group operation via a linear operator $h$. 

We next give GM formulations of general versions of common group convolutional network choices and operations: convolutions, small kernels, strides and pooling.

\textbf{Group convolution using group matrices.} We can re-express group convolution in the language of group matrices and group diagonals as follows.
Consider the functions $\phi,\psi:G\to\mathbb R$ as vectors in $\overrightarrow\phi, \overrightarrow\psi\in \mathbb R^G$, where the vector entries are indexed by $G$ -- e.g. $\overrightarrow\phi_g:=\phi(g)$. 
In this case, the linear transformation $\overrightarrow\psi\mapsto\overrightarrow{\phi\star\psi}$ is expressed by a simple matrix. 
Moreover, it is sufficient to combine group diagonals $B_g$ with coefficients encoded in the entries of $\psi_g$ to define convolution as:
\begin{equation}\label{eq:gmconv}
    \mathsf{Conv}_\psi\overrightarrow\phi:=\overrightarrow{\phi\star\psi}=\sum_{g\in G} \phi(g) B_g \overrightarrow \psi.
\end{equation}
The verification of \eqref{eq:gmconv} is straightforward: using definitions \eqref{eq:gconv} and \eqref{eq:gdiag}, we have, for all $x\in G$, 
\vspace{-1mm}
\begin{eqnarray}
\left[\overrightarrow{\phi\star\psi}\right]_x\!\!\!\!\! &:= &\!\!\!\sum_{g,h: gh=x}\!\!\!\!\phi(g)\psi(h) = \sum_{g,h}\phi(g)\psi(h)\delta(gh=x)\nonumber \\
&=&\!\!\!\!\! \sum_{g}\phi(g)\left[B_g\overrightarrow\psi\right]_x\!\!\!\!\! =\! \left[\sum_g \phi(g)B_g\overrightarrow\psi\right]_x\!\!\!\!\!.
\end{eqnarray}
\vspace{-2mm}


\textbf{Choosing small kernels.} In classical CNNs, kernels have support in a $k\times k$-neighborhood of the origin for small values of $k$: 
kernels are nonzero at pixels near the origin.
The analogue of this choice for general finite groups $G$ is to fix a word distance\footnote{For a finite group $G$ w.r.t the generating set $S$, assign each edge of the Cayley graph a metric of length 1. Then the distance between $g,h \in G$ equals the shortest path length in the Cayley graph from vertex $g$ to vertex $h$} and to restrict kernel coefficients $\overrightarrow \phi$ to be zero on elements of $G$ at a distance larger than $k$ from the origin. 
Consequently, if $N_k$ is the number of elements in the radius-$k$ neighborhood of the identity in $G$, then the sum from \eqref{eq:gmconv} only includes $N_k$ non-zero summands, which is typically $N_k$ orders of magnitude smaller than the cardinality $|G|$. 
Note that we are never required to compute the $B_g$ matrices except for these $N_k$ choices of $g$. 
This permits an efficient implementation analogous to classical CNNs. 

\textbf{Pooling operations: Group diagonals can be restricted to subgroups.}
Pooling and stride operations map input channels, written as functions $\psi:G\to \mathbb R$, to outputs of the form $\psi': H\to\mathbb R$, indexed by a subgroup $H\subseteq G$. 
Convolution restricts naturally to $H$, a property that can be formulated in terms of group diagonals. 
The group diagonals $B_h^H$ ($|H|\times|H|$ matrices with group $H$) acting over $\overrightarrow\psi'$ can be obtained from the $B_h^G$ ($|G|\times|G|$ matrices with group $G$ corresponding to elements $h\in H\subseteq G$) by removing rows and columns indexed by elements of $G\setminus H$. 
In particular, group diagonals are mapped to group diagonals under this operation, as a consequence of the following lemma:
\begin{lemma}\label{lem:subgroupdiag}
If $h\in H$, the rows of $G$-group diagonal matrix $B_h$ labeled by elements of $H$, have entries $1$ only in the columns whose index labels belong to $H$.
\end{lemma}
\vspace{-1mm}
All proofs are relegated to the appendix. 

\textbf{Pooling operations on finite groups.} As mentioned above, we define the pooling operation via a subgroup $H\subset G$, where pooling is done over subsets of size $n=|G|/|H|$. 
We now explicitly describe the operation. 
Consider right cosets $H g_1,\dots, H g_n$. 
For each coset $Hg_i$, we select a coset representative at the closest word distance from the origin: $\bar g_i\in \mathsf{argmin}_{h\in H} \mathsf{dist}(hg_i, e)$. 
Note that for $h\in H$, the sets $P_h:=\{h\bar g_i:\ 1\le i\le n\}$ again form a partition of $G$.
Then, for a choice $\square\in\{\max,\mathsf{aver},\dots\}$, the $(H,\square)$-pooling of a function $\psi:G\to\mathbb R$ is a function over $H$ given by
\begin{equation}\label{eq:gpool}
\mathsf{Pool}_H^{\square}(\psi)(h) := \square\{\psi(g'):\ g'\in P_h\}.
\end{equation}

\begin{example}
    If the group is $G=C_N\times C_N$ with generators denoted as $(1,0), (0,1)$, and assuming that we can factorize $N=mn$ with $m,n\geq 2$, then we can focus on the subgroup $H\subseteq G$ isomorphic to $C_n\times C_n$, generated by elements $(m, 0), (0,m)\in G$. In this case, cosets of $H$ are indexed by $0\le i,i'<m$:
    \[
    H\cdot (i,i') = \{(km+i,k'm+i')\}_{k,k'=0}^n.
    \]
    We can then directly take values $(i,i')$ as above as the closest-to-identity representatives $\overline g_\alpha$. With this notion, the  $(H,\mathsf{aver})$-pooling operation maps input a signal $\{\psi(a,a')\}_{a,a'=0}^N$ to 
    \[
    \Big\{\psi'(k,k'): = \frac{1}{m^2}\sum_{0\le i,i'<m}\psi(k+i,k'+i')\Big\}_{k,k'=0}^n.
    \vspace{-1mm}
    \]
    
    \textbf{Remark.} In practice, successive layers to be applied to the output of a $\mathsf{Pool}^{\square}_H$-layer should use group matrices associated with group $H$. 
    Such group matrices can be constructed as follows: take group diagonals $B_g$ of $C_N\times C_N$, with $g$ of the form $g=(km,k'm)$ only, of which we remove rows and columns whose index is not a multiple of $m$.
\end{example}

\textbf{Implementing stride via subgroups.} Increasing stride after convolution operations allows the user to manipulate signal sizes across the architecture. 
If the input is indexed by finite group $G$, stride can be defined by restricting the output to indices from a subgroup $H$. For $\overrightarrow \psi\in\mathbb R^G$ we set $\mathsf{Conv}_\phi^{H}\overrightarrow \psi:=\sum_{g\in G}\psi(g) B_g^H\overrightarrow\psi\in\mathbb R^H$,
in which $B_g^H$ is the $|H|\times|G|$ matrix obtained by removing rows of $B_g$ that are labeled by $G\setminus H$.

\textbf{Obtaining group matrices of complex groups from simpler ones.} In many applications, groups are products or semi-direct products of cyclic or permutation groups. 
For any finitely generated group, one can produce the group matrix using the group operation and an algorithm for bringing group elements to canonical form. 
However, this can be time-consuming: one would ideally use the group matrices of individual components and produce a group matrix of the larger group directly. 
To achieve this, we use the following procedure.
Let $G, H$ be two groups, and consider the product group $G\times H$. 
Then, if $g\in G, h\in H$, and $B^G_g, B^H_h$ are the corresponding group diagonals, a group diagonal in $G\times H$ for element $(g,h)$ is given by taking the Kronecker product
\begin{equation}\label{eq:gdprod}
 B^{G\times H}_{(g,h)}=B^G_g\otimes B_h^H.
\end{equation}
Consider a semi-direct product $G\rtimes_\phi H$. 
As usual, it requires defining a group homomorphism $\phi:H\to \mathrm{Aut}(G)$, and then group operation is defined by $(g,h)\cdot(g',h'):=(g\ \phi_h(g'), hh')$. 
In this case, $g\mapsto \phi_h(g)$ induces a permutation of $G$, and thus we can associate to it a $|G|\times |G|$ permutation matrix $P_h$. 
Then it is direct to check that $B_{(g,h)}^{G\rtimes_\phi H} = (P_hB_g^G)\otimes B_h^H$.

We now have fully described generalizations of the classical CNN operations to general discrete groups via group matrices. Additionally, we demonstrated that direct or semi-direct product groups of cyclic or permutation groups permit efficient computation of their group matrices. 

\textbf{Closure of group matrices under elementary operations.} As reviewed in \S\ref{sec:DSreview}, closure properties of LDR matrices were key to their use in NN compression~\citep{thomas2018learning}. Before we introduce error control and approximate equivariance, we note the following simple closure properties, proved in \S\ref{app:proofgmalgebra}.
\looseness=-1
\begin{prop}\label{prop:gmalgebra}
    If $M, M'$ are group matrices for group $G$, then $M^T, M^{-1}, MM'$ are also group matrices for group $G$. If $N$ is a group matrix for group $H$, then the Kronecker product $M\otimes N$ is a group matrix for the direct product group $G\times H$.
\end{prop}
\vspace{-1mm}
\section{Implementing Approximately Equivariant GM-CNNs}\label{sec:approxequiv}
\vspace{-1mm}
A number of researchers have recently argued, both empirically~\citep{FinziSoftEquivariance, romero2022learning,ouderaa2022relaxing, wang2022approximately} and theoretically~\citep{Huang2023ApproximatelyEG, PetracheApproximate}, that imposing a hard equivariance constraint in NNs can be detrimental to accuracy; practitioners can benefit by relaxing equivariance at minimal computational cost. 
Our GM-based formalism can be considered as only using a superposition of constant-valued diagonals $\phi(g) B_g$ as dictated by the formula \eqref{eq:gmconv}. We now propose a generalization that allows learnable weights beyond Section \ref{sec:gms}, thus permitting a principled implementation of approximately equivariant NNs.

\textbf{General matrices in group-diagonal basis.} We first show how general matrices can be written in a group-diagonal basis. Consider a general $|G|\times|G|$-matrix $M$, and let $B_g, g\in G$ be the group diagonals \eqref{eq:gdiag}. 
Then we can always encode the entries from $M$ via $|G|$-dimensional arrays $F_g, g\in G$:
\vspace{-1mm}
\begin{equation}\label{eq:genmatrix}
\begin{split}
M &= \sum_{g\in G}\mathrm{diag}(F_g) B_g, \\
&\quad \text{where for } h\in G, \quad (F_g)_h:=M_{h, hg^{-1}}.
\end{split}\vspace{-2mm}
\end{equation}
The validity of expression \eqref{eq:genmatrix} follows directly from definition \eqref{eq:gdiag}: the $(h,h')^{th}$ entry of $B_g$ is $\delta(h=gh')$, and thus is nonzero (in fact $=$1) only if $g=h(h')^{-1}$. Using the definition \eqref{eq:genmatrix} of $F_g$,
\[
\begin{split}
\Big[\sum_{g\in G} \mathrm{diag}(F_g) B_g\Big]_{h,h'} &= \sum_{g\in G}(F_g)_h\ \delta(h=gh') \\
&= (F_{h(h')^{-1}})_h \\
&= M_{h,h(h(h')^{-1})^{-1}} = M_{h,h'}.
\end{split}
\]
and thus all the entries of the two sides of \eqref{eq:genmatrix} coincide.

We can obtain the coordinates $F(M)$ associated with any linear operation $\overline\psi\mapsto M\overline\psi$, by shuffling the entries of $M$ and reducing the matrix $F(M)$, with rows and columns labeled by $G$, and whose rows are the coefficients $F_g, g\in G$. We can interpret the columns of $F(M)$ as relative coefficients multiplied at position $g$ of the group, and row entries as parameterizing the relative position:
\vspace{-1mm}
\[
    \psi_{out}(\bar g) = \sum_{g\in G}\psi_{in}(\bar gg^{-1})[F(M)]_{\bar g,g}.
    \vspace{-1mm}
\]
With the above, we get a formulation of a more general convolution (learned via $M$), permitting approximate equivariance, but expressed in terms of group matrices. We now connect this formulation to the theory of displacement structures and show that it corresponds to a LDR implementation.

\textbf{Displacement operator for general groups.} Let $\mathbf{b}$ represent the vector such that its $i^{th}$ entry is the constant value associated with row $i$. Suppose if the matrix $F(M)$ has the form $F(M)=\mathbf 1 \otimes \mathbf b$, then it is equivalent to $M$ being a group matrix, and the convolution described above reduces to an exact convolution as in \eqref{eq:gmconv}. This property of $F(M)$ can be tested by taking the following difference: 
\begin{equation}\label{eq:dr}
\begin{split}
\mathsf{D}(M) &= \mathsf{D}_P(M):= F(M)-PF(M), \\
&\quad \text{where $P(x_1,\dots, x_N)=(x_2,\dots,x_N, x_1)$}.
\end{split}
\end{equation}
Then $M$ is a group matrix (and thus encodes a group convolution) if and only if $\mathsf{D}(M)=0$.

In fact, the same property $\mathsf{D}(M)=0$ holds even if in defining $\mathsf{D}$ we apply \emph{a different cyclic permutation for each row}. More explicitly, for each $g\in G$ we select a cyclic permutation $\sigma_g\in\mathsf{Perm}(G)$, and then set $\vec P:=(\sigma_g)_{g\in G}$ and define entrywise
\begin{equation}\label{eq:dr2}
    \left[\mathsf{D}_{\vec P}(M)\right]_{g,g'}:=\left[F(M)\right]_{g,g'} - \left[F(M)\right]_{g,\sigma_g(g')}.
\end{equation}

\textbf{Displacement dimension and rank.} When we increase expressivity by allowing a controlled error to equivariance, a natural metric for this control is via the dimension of the space of allowed matrices $\mathsf{D}(M)$ within a model. We define the \emph{displacement dimension} of a subset $\mathcal M\subseteq\mathbb R^{|G|\times |G|}$ as: 
\begin{equation}\label{eq:dimd}
    \mathrm{dim}_{\mathsf{D}}(\mathcal M):=\dim_{\mathbb R}(\mathsf{Span}(\{\mathsf{D}(M):\ M\in\mathcal M\})).
\end{equation}
Note that $\mathsf{dim}_{\mathsf{D}}(\mathcal M)$ does not depend on the choice of $\vec P$ in \eqref{eq:dr2}, as the dimension considered in \eqref{eq:dimd} can be computed by summing the dimensions of spans row by row, and that row spans for $\mathrm{D}_{\vec P}(M), M\in\mathcal M$ do not depend on the choice of permutations $\sigma_g$.

Another metric for measuring the discrepancy of a matrix $M$ from being a group matrix, is the \emph{displacement rank} of $M$, defined as
\begin{equation}\label{eq:disprank}
    \mathsf{DR}(\mathcal M):=\mathrm{rank}(\mathsf{D}(M)).
\end{equation}
For further discussion about this notion of displacement rank and connections to the classical generalization of displacement rank as introduced in \cite{gader1990displacement, waterhouse1992displacement}, see Appendix \ref{app:ldr}.

\textbf{Low displacement rank implementations.} In order to introduce errors to equivariance, we use kernels, encoded as matrices $M$ whose displacement matrices $\mathsf{D}(M)$ have low rank. A simple choice is to add a matrix with $r$ learnable vector columns $\mathbf a_{g_i}, 1\le i\le r$ to the a group matrix $M$:
\begin{equation}\label{eq:ldr1}
    F(M)= \mathbf 1\otimes \mathbf b + \sum_{i=1}^d \mathbf a_{g_i}\otimes \mathbf 1.
\end{equation}
Then it is direct to verify that $\mathrm{rank}(\mathsf{D}(M))=\mathrm{dim}(\mathrm{Span}(\mathbf a_{g_1},\dots,\mathbf a_{g_r}))\leq r$, and that if $\mathcal M$ is the space of matrices of the form \eqref{eq:ldr1} then $\mathrm{dim}_{\mathsf{D}}(\mathcal M)=|G| r$. 

\textbf{Quantifying equivariance error under elementary operations.} As a counterpart to the closure properties of Prop. \ref{prop:gmalgebra} for group matrices, it is interesting to quantify control on how a bound on the error to equivariance (or to ``being a group matrix'') behaves under the same operations.

A first approach is to use displacement-based structural metrics such as $\mathrm{DR}(\mathcal M)$ or $\mathrm{dim}_{\mathrm{D}}(\mathcal M)$ and ask if they behave the same way on composite operations between matrices of the same class. Such a control for the case of classical LDR is available in \cite[Prop. 1]{thomas2018learning}. However we found that the natural displacement operator $\mathsf{D}$, in general, has a very complicated behavior under matrix products, and we were not able to extend the natural bounds for deterioration under multiplication from \cite{thomas2018learning}. 

On the other hand, rather than demanding structural/algebraic control, a quantitative control over the error to equivariance may be more useful in practice, for equivariance error bounds. A natural quantification is to measure the distance of a matrix from the set of $G$-group matrices: 
\begin{equation}\label{eq:distgm}   
    \mathrm{dist}(M, \mathcal{GM}):=\min\{ \|M - M_0\|:\ M_0\in \mathcal{GM}\},
\end{equation}
in which $\|\cdot\|$ is Frobenius norm for matrices. We then have the following properties, proved in \S\ref{app:proofdistancebds}:
\begin{prop}\label{prop:distancebds}
    Let $M,M'$ be $|G|\times|G|$-matrices and let $\mathcal{GM}=\mathcal{GM}^G$ be the set of $G$-group matrices. For a second group $H$, let $N$ be an $|H|\times|H|$-matrix and $\mathcal{GM}^H, \mathcal{GH}^{G\times H}$ be the set of group matrices with group $H, G\times H$ respectively. Then
    \vspace{-1mm}
    \begin{enumerate}
        \item $\mathrm{dist}(M, \mathcal{GM})=\mathrm{dist}(M^T,\mathcal{GM})$.
        \item $\mathrm{dist}(MN,\mathcal{GM})\leq \\ \max\{\|M\|, \|N\|\}\left(\mathrm{dist}(M,\mathcal{GM}) + \mathrm{dist}(M',\mathcal{GM})\right)$.
        \item $\mathrm{dist}(M\otimes N, \mathcal{GM}^{G\times H})\leq \\ \max\{\|M\|, \|N\|\}(\mathrm{dist}(M,\mathcal{GM}^G) + \\ \mathrm{dist}(N,\mathcal{GM}^H))$.
    \end{enumerate}
\end{prop}


Finally, as outlined in Section \ref{sec:intro}, our framework extends beyond discrete groups to their homogeneous spaces. Due to space constraints, we describe the extension to homogeneous spaces in Appendix~\ref{app:homogneous}. Additionally, we demonstrate that the framework extends to finitely supported data on infinite discrete groups in \S\ref{app:padding}. The error to perfect equivariance for this general setup is quantified in \S\ref{app:errortoeq}.

\section{Experimental Results}
\vspace{-2mm}
\textbf{GMConv and GMPool operations:} Our architecture is built around GMConv and GMPool operations to handle group-based interactions and symmetry-preserving pooling. GMConv uses group matrices to define convolutional kernels, with the neighborhood parameter controlling the extent of local interactions based on the word distance (Section~\ref{sec:gms}) to implicitly set the kernel size. For cyclic groups, a neighborhood radius of 1 results in 3 group entries $[x-1,x, x+1 ]$, while for direct product of two such groups, it results in 9 learnable parameters. Generally, a GMConv layer with a neighborhood size $k$ comprises $(2k+1)\times(2k+1)\times  \text{number of channels}$ learnable parameters. In error addition experiments, we double the parameters by introducing an additional error matrix with same structure as the initial group matrix, allowing to learn approximate equivariance. GMPool (Section ~\ref{sec:gms}) uses predefined group structures, constructed using the Kronecker products of group elements and subgroup cosets. These cosets, along with their corresponding tensor indices, are precomputed and utilized during the forward pass of GMPool to select elements for max or mean pooling. Following \cite{wang2022approximately}, we also provide equivariance error analysis (Appendix~\ref{app:error_analysis}).

\subsection{Dynamics Prediction}
\vspace{-1em}
We evaluate our framework on two dynamics prediction tasks: the smoke plumes task (Plumes) and the Jet Flow task (JetFlow) as described in \cite{wang2022approximately}. For Plumes, each model takes sequences of $64 \times 64$ crops of a smoke simulation generated by PhiFlow \citep{phiflow} as input to predict the velocity field for the next time step. The evaluation is conducted under two settings: ``Future'', where we evaluate on the same portion of the simulation, but predict future time steps which are not included in training. The second setting is ``Domain'' where the evaluation is done on the same time steps but at different spatial locations. The data are collected from simulations with different inflow positions and buoyant forces. The JetFlow task comprises 24 subregions of jet flows sized $62 \times 23$, as described in \cite{wang2022approximately} following a similar evaluation protocol. Our method is not inherently steerable; but we still compare it against various steerable baselines despite putting it at a disadvantage. The baselines include a standard MLP and a CNN, as well as an E2-CNN (Equiv) \citep{e2cnn}, and two approximately equivariant networks namely RPP \citep{finzi2021residual} and LIFT \citep{wang2021equivariant}. GM-CNN architecture choices and hyperparameter details are provided in Appendix~\ref{app:experimental_details}.

\begin{table*}[ht]
    \centering
    \caption{RMSE on Jet Flow dataset. See text for details~\citep{wang2022approximately}. FLOPs are $\times 10^{10}$ (lower the better).}
    \label{tab:results_jet}
    \begin{tabular}{c@{\hskip 0.15cm}|c@{\hskip 0.15cm}c@{\hskip 0.15cm}c@{\hskip 0.15cm}c@{\hskip 0.15cm}|c@{\hskip 0.15cm}c@{\hskip 0.15cm}c@{\hskip 0.15cm}c@{\hskip 0.15cm}|c@{\hskip 0.15cm}c@{\hskip 0.15cm}c@{\hskip 0.15cm}}
        \toprule
        \textbf{Model} & \multicolumn{4}{c@{\hskip 0.08cm}|}{\textbf{Translation}} & \multicolumn{4}{c@{\hskip 0.08cm}|}{\textbf{Rotation}} & \multicolumn{3}{c@{\hskip 0.08cm}}{\textbf{Scaling}} \\
        \midrule
        & \texttt{Conv} & \texttt{Lift} & \texttt{RGroup} & \texttt{GM-CNN} & \texttt{E2CNN} & \texttt{Lift} & \texttt{RSteer} & \texttt{GM-CNN} & \texttt{Rpp} & \texttt{RSteer} & \texttt{GM-CNN} \\
        \midrule
        Future & 0.22 & 0.17 & 0.15 & 0.15 & 0.21 & 0.18 & 0.17 & 0.16 & 0.16 & 0.14 & 0.15 \\
               & {\tiny $\pm$0.06} & {\tiny $\pm$0.02} & {\tiny $\pm$0.00} & {\tiny $\pm$0.01} & {\tiny $\pm$0.02} & {\tiny $\pm$0.02} & {\tiny $\pm$0.01} & {\tiny $\pm$0.01} & {\tiny $\pm$0.06} & {\tiny $\pm$0.01} & {\tiny $\pm$0.01} \\
        Domain & 0.23 & 0.18 & 0.16 & 0.17 & 0.27 & 0.21 & 0.16 & 0.18 & 0.16 & 0.15 & 0.17 \\
               & {\tiny $\pm$0.06} & {\tiny $\pm$0.04} & {\tiny $\pm$0.01} & {\tiny $\pm$0.01} & {\tiny $\pm$0.03} & {\tiny $\pm$0.04} & {\tiny $\pm$0.01} & {\tiny $\pm$0.02} & {\tiny $\pm$0.07} & {\tiny $\pm$0.00} & {\tiny $\pm$0.01} \\
        \midrule
        Params & {\fontsize{9}{10}\selectfont 51548} & {\fontsize{9}{10}\selectfont 1994818} & {\fontsize{9}{10}\selectfont 53798} & {\fontsize{9}{10}\selectfont 26325} & {\fontsize{9}{10}\selectfont 107136} & {\fontsize{9}{10}\selectfont 1915872} & {\fontsize{9}{10}\selectfont 961538} & {\fontsize{9}{10}\selectfont 29583} & {\fontsize{9}{10}\selectfont 1421832} & {\fontsize{9}{10}\selectfont 6742530} & {\fontsize{9}{10}\selectfont 26325} \\
        \midrule
        FLOPs & {\fontsize{9}{10}\selectfont 0.007} & {\fontsize{9}{10}\selectfont 0.023} & {\fontsize{9}{10}\selectfont 0.007} & {\fontsize{9}{10}\selectfont 0.006} & {\fontsize{9}{10}\selectfont 0.185} & {\fontsize{9}{10}\selectfont 0.015} & {\fontsize{9}{10}\selectfont 0.922} & {\fontsize{9}{10}\selectfont 0.008} & {\fontsize{9}{10}\selectfont 0.141} & {\fontsize{9}{10}\selectfont 0.024} & {\fontsize{9}{10}\selectfont 0.006} \\
        \bottomrule
    \end{tabular}
\end{table*}

\begin{table*}[ht]
    \centering
    \caption{RMSE on the Plumes dataset. See text for details~\citep{wang2022approximately}. FLOPs are $\times 10^{10}$.}
    \label{tab:results_smoke}
    \resizebox{\textwidth}{!}{ 
    \large
    \begin{tabular}{|c|c|c|c|c|c|c|c|c|c|c|c|}
        \toprule
        \multicolumn{2}{|c|}{Model} & \texttt{MLP} & \texttt{Conv} & \texttt{Equiv} & \texttt{Rpp} & \texttt{Combo} & \texttt{Lift} & \texttt{RGroup} & \texttt{RSteer} & \texttt{GM-CNN} \\
        \midrule
        {Translation} & F & 1.56\scriptsize{$\pm$0.08} & ----- & 0.94\scriptsize{$\pm$0.02} & 0.92\scriptsize{$\pm$0.01} & 1.02\scriptsize{$\pm$0.02} & 0.87\scriptsize{$\pm$0.03} & \textbf{0.71\scriptsize{$\pm$0.01}} & ----- & 0.81\scriptsize{$\pm$0.02} \\
                                     & D & 1.79\scriptsize{$\pm$0.13} & ----- & 0.68\scriptsize{$\pm$0.05} & 0.93\scriptsize{$\pm$0.01} & 0.98\scriptsize{$\pm$0.01} &  0.70\scriptsize{$\pm$0.00} & \textbf{0.62\scriptsize{$\pm$0.02}} & ----- & 0.74\scriptsize{$\pm$0.01} \\
                                     & Params & \texttt{8678240} & ----- & \texttt{1821186} & \texttt{7154902} & \texttt{3683332} & \texttt{8235362} & \texttt{1921158} & ----- & \texttt{19267} \\
                                     & FLOPs & \texttt{0.001} & ----- & \texttt{0.745} & \texttt{0.002} & \texttt{1.492} & \texttt{0.145} & \texttt{0.756} & ----- & \texttt{0.003} \\
        \midrule
        {Rotation} & F & 1.38\scriptsize{$\pm$0.06} & 1.21\scriptsize{$\pm$0.01} & 1.05\scriptsize{$\pm$0.06} & 0.96\scriptsize{$\pm$0.10} & 1.07\scriptsize{$\pm$0.00} & 0.82\scriptsize{$\pm$0.08} & 0.82\scriptsize{$\pm$0.01} & \textbf{0.80\scriptsize{$\pm$0.00}} & 0.93\scriptsize{$\pm$0.02}\\
                                    & D & 1.34\scriptsize{$\pm$0.03} & 1.10\scriptsize{$\pm$0.05} & 0.76\scriptsize{$\pm$0.02} & 0.83\scriptsize{$\pm$0.01} & 0.82\scriptsize{$\pm$0.02} &  0.68\scriptsize{$\pm$0.09} & 0.73\scriptsize{$\pm$0.02} & \textbf{0.67\scriptsize{$\pm$0.01}} & 0.79\scriptsize{$\pm$0.02}\\
                                    & Params & \texttt{8678240} & \texttt{1821186} & \texttt{1198080} & \texttt{5628298} & \texttt{431808} & \texttt{1801748} & \texttt{1883536} & \texttt{7232258} & \texttt{19267} \\
                                    & FLOPs & \texttt{0.001} & \texttt{0.745} & \texttt{2.965} & \texttt{3.966} & \texttt{0.585} & \texttt{2.955} & \texttt{3.087} & \texttt{1.986} & \texttt{0.003} \\
        \midrule
        {Scaling} & F & 2.40\scriptsize{$\pm$0.02} & 0.83\scriptsize{$\pm$0.01} & 0.75\scriptsize{$\pm$0.03} & 0.81\scriptsize{$\pm$0.09} & 0.78\scriptsize{$\pm$0.04} & 0.85\scriptsize{$\pm$0.01} & 0.76\scriptsize{$\pm$0.04} & \textbf{0.70\scriptsize{$\pm$0.01}} & 0.79\scriptsize{$\pm$0.02}  \\
                                     & D & 1.81\scriptsize{$\pm$0.18} & 0.95\scriptsize{$\pm$0.02} & 0.87\scriptsize{$\pm$0.02} & 0.86\scriptsize{$\pm$0.05} & 0.85\scriptsize{$\pm$0.01} & 0.77\scriptsize{$\pm$0.02} & 0.86\scriptsize{$\pm$0.12} & \textbf{0.73\scriptsize{$\pm$0.01}} & 0.82\scriptsize{$\pm$0.01}\\
                                     & Params & \texttt{8678240} & \texttt{1821186} & \texttt{1744774} & \texttt{3966984} & \texttt{1059270} & \texttt{2833558} & \texttt{1275266} & \texttt{2427394} & \texttt{19267} \\
                                     & FLOPs & \texttt{0.001} & \texttt{0.745} & \texttt{0.368} & \texttt{0.507} & \texttt{0.326} & \texttt{0.016} & \texttt{0.132} & \texttt{0.041} & \texttt{0.003} \\
        \bottomrule
    \end{tabular}
    }
    \vskip -0.1in
\end{table*}




\textbf{Results on JetFlow:} Table \ref{tab:results_jet} shows GM-CNN achieves top performance (tied with RGroup) on the translation task, using the fewest parameters (26,325). Standard Conv and Lift use 51,548 and 1,994,818 parameters respectively. For the rotation task, GM-CNN performs slightly better than ECNN while using only about 10\% of the parameters (29,583 vs 304,128). It also surpasses Lift and RSteer. Finally, in the scaling task, GM-CNN's performance is comparable to the best performers (Rpp and RSteer) while using only a fraction of their parameters (26,325 compared to 1,421,832 for Rpp and 6,742,530 for RSteer). These results highlight the competitive performance of our method, despite its lack of inherent steerability.

\textbf{Results on Plume:} As seen in Table \ref{tab:results_smoke}, GM-CNN consistently uses the fewest parameters (19,267) across all tasks while achieving competitive results. For the translation task, GM-CNN's performance is close to the best-performing RGroup, despite using 25 times fewer parameters. In Rotation, it competes well with top performing methods like Lift and RSteer while maintaining its dramatic parameter efficiency advantage. Lastly, for scaling, GM-CNN is competitive, with RSteer achieving only slightly better results at the cost of 350 times more parameters.

\begin{table*}[ht]
    \centering
    \resizebox{\textwidth}{!}{
    \small
    \begin{tabular}{cccccccc}
    \toprule
         Method $\uparrow$ & \textbf{M-bg-rot} & \textbf{M-noise} & \textbf{CIFAR-10} & \textbf{NORB} & \textbf{SmallNORB} & \textbf{Rect} & \textbf{Rot-MNIST} \\
         \midrule
         LDR-TD (r=1) & 45.81 & 78.45 & 45.33 & 62.75 & \textcolor{red}{83.23} & 98.53 & \textcolor{red}{79.82}\\
         & 14122 & 14122 & 18442 & 14342 & 14122 & 14122 & 14122\\
         \midrule
         GM (n=3)  & 30.07 & 80.14 & 58.31 & 54.63 & 79.01 & 99.60 & 83.78\\
         & 12483 & 13441 & 13681 & 12967 & 12725 & 12483 & 48.7k \\
         \midrule
         GM (n=3, E) & 49.07 & 82.55 & \textcolor{blue}{\textbf{58.29}} & \textcolor{blue}{\textbf{70.59}} & \textcolor{blue}{\textbf{85.22}} & 99.31 & 83.26\\
         & 24734 & 25213 & 25213 & 24734 & 24546 & 24244 & 24734 \\
         \midrule
         GM (n=1, P + E) & \textcolor{blue}{\textbf{55.29}} & \textcolor{blue}{\textbf{90.20}} & \textcolor{red}{58.07} & \textcolor{red}{67.84} & 80.90 & \textcolor{blue}{\textbf{99.89}} & 78.77\\
         & \textbf{5915} & \textbf{5915} & \textbf{6003} & \textbf{5687} & 5573 & \textbf{5915} & 5915\\
         \midrule
         GM (n=2, P + E) & \textcolor{red}{53.94} & \textcolor{red}{88.84} & 55.10 & 67.72 & 78.61 & \textcolor{red}{99.86} & 79.06\\
         & \textbf{8701} & \textbf{8701} & 8819 & 8503 & 8359 & \textbf{8701} & 8731\\
         \midrule
         GM (n=3, P + E) & 53.88 & 84.92 & 57.14 & 66.99 & 83.02 & 99.67 & \textcolor{blue}{\textbf{84.50}}\\
         & 14747 & 14747 & 14482 & 14519 & 14482 & 14747 & 14747\\
         \midrule
         Low-rank[67](r=4) & 35.67 & 52.25 & 32.28 & 43.66 & 78.05 & 87.48 & 54.87 \\
         & 14122 & 14122 & 18442 & 14342 & 6916 & 7842 & 14122 \\
         \midrule
         Fastfood[68] & 38.13 & 63.55 & 39.64 & 59.02 & 73.38 & 89.81 & 58.14 \\
         & 10202 & 10202 & 13322 & 9222 & 5380 & 10202 & 10202 \\
        \midrule
         Circulant[69] & 34.46 & 65.35 & 34.28 & 46.45 & 71.23 & 88.92 & 52.22 \\
         & 8634 & 8634 & 11274 & 7174 & 3456 & 8634 & 8634\\
         \bottomrule
         \smallskip
         
    \end{tabular}}
    \caption{Test accuracy and parameter count for GM-CNNs. In the ``Method" column, \textbf{n} denotes the neighbourhood size, \textbf{P} suggests the use of pooling and \textbf{E} suggests error addition. The best accuracy is in \textcolor{blue}{blue}, and the second best in \textcolor{red}{red}. If GM-CNN achieves best or second-best accuracy \textbf{with fewer parameters}, then we mark the number of parameters in \textbf{bold}. Clearly, we see that GM-CNN methods consistently provide the most accurate and usually the lowest parameter count. In cases where the best GM-CNN model has the best accuracy, but not the lowest parameter count, the second best GM-CNN model has comparable accuracy, but significantly lower parameter count.}
    \label{tab:test_acc}
\end{table*}

\subsection{Comparison with structured matrix baselines~\citep{sindhwani2015structured, thomas2018learning}}

We evaluate GM-CNNs on a variety of image datasets and compare them against a set of competitive baselines---the methods reported in \cite{thomas2018learning} still remain amongst the most competitive. We compare our methods on two dimensions: accuracy of the models and the total number of parameters. In addition to datasets from~\cite{thomas2018learning}, we further consider the SmallNORB~\citep{lecun2004learning}, Rotated MNIST, and Rectangles datasets~\citep{larochelle2007empirical}. SmallNORB is a condensed version of the NORB dataset~\citep{lecun2004learning}, and is specifically tailored for object recognition tasks focusing on shape. The Rotated MNIST dataset comprises MNIST data samples that have been randomly rotated, offering a variation that lacks the noise typically introduced in the MNIST-bg-rot~\citep{larochelle2007empirical} dataset. The Rectangles dataset, a compact binary classification dataset, distinguishes rectangles based on whether their width or height is greater. For these image classification tasks, the input dimensions range from $24\times24$ to $32\times32$. While small, these datasets capture a wide range of variation, but more importantly, they permit comparison with existing structured matrix baselines. 

\textbf{Results.} GM-CNN variants demonstrate strong performance across all datasets (Table~\ref{tab:test_acc}). However, our approach, combining error addition and pooling with a neighborhood parameter set to 1, achieves the best results on the MNIST-bg-rot, MNIST-noise, and Rectangles datasets with minimal parameters. On CIFAR-10, all GM variants outperform other methods. For NORB and SmallNORB, GM-CNN with a neighborhood size of 3 and error addition shows superior performance. For Rotated MNIST, the pooling and error addition approach with a neighborhood size of 3 leads in performance. These performances are achieved with significantly fewer parameters. Note that the numbers for CIFAR-10 are much lower than state-of-the-art results, which we attribute to the significantly lower parameter counts. 

\vspace{-3mm}
\section{Conclusion}
\vspace{-2mm}
In this paper, we presented a novel formalism (GM-CNNs) for constructing equivariant networks for general discrete groups using group matrices, generalizing all the elementary operations of classical CNNs. GM-CNNs employ the use of a novel family of symmetry-based structured matrices, also developed via our formalism, which facilitates the construction of lightweight equivariant NNs. Further, we presented a principled implementation of approximately equivariant group CNNs using our formalism. Connecting group matrices to classical displacement structure theory, we provide a generalization of the theory from cyclic groups to general discrete groups. Moving beyond discrete groups, we provide an extension of GM-CNNs to homogeneous spaces and infinite discrete groups. Finally, we tested our proposed formalism on a variety of different tasks, and show that GM-CNNs can be consistently competitive, while being significantly more parameter efficient, compared to approximately equivariant NNs and structured matrix-based frameworks. For future work, it would be interesting to extend our formulation for continuous groups, and enabling our setup to be steerable. Further, exploring the group tensorization operations proposed in \cite{derksen2016nuclear} (e.g. theorem 1.1), could help improve the scalability of our method.

\subsubsection*{Acknowledgements}
This work is supported by the Deep Skies Community (deepskieslab.com), which helped in bringing together the authors. MP thanks the Centro Nacional de Inteligencia Artificial in Chile, as well as Fondecyt Regular grant number 1210462 titled “ Rigidity, stability and uniformity for large point configurations” for support.\\
Notice: This work was produced by FermiForward Discovery Group, LLC under Contract No. 89243024CSC000002 with the U.S. Department of Energy, Office of Science, Office of High Energy Physics. The United States Government retains and the publisher, by accepting the work for publication, acknowledges that the United States Government retains a non-exclusive, paid-up, irrevocable, world-wide license to publish or reproduce the published form of this work, or allow others to do so, for United States Government purposes. The Department of Energy will provide public access to these results of federally sponsored research in accordance with the DOE Public Access Plan (http://energy.gov/downloads/doe-public-access-plan).





\bibliography{bibliography}




\section*{Checklist}



 \begin{enumerate}

 \item For all models and algorithms presented, check if you include:
 \begin{enumerate}
   \item A clear description of the mathematical setting, assumptions, algorithm, and/or model. [Yes]
   \item An analysis of the properties and complexity (time, space, sample size) of any algorithm. [Not Applicable]
   \item (Optional) Anonymized source code, with specification of all dependencies, including external libraries. [Yes. Available at: \href{https://github.com/kiryteo/GM-CNN}{https://github.com/kiryteo/GM-CNN}]
 \end{enumerate}

 \item For any theoretical claim, check if you include:
 \begin{enumerate}
   \item Statements of the full set of assumptions of all theoretical results. [Yes]
   \item Complete proofs of all theoretical results. [Yes]
   \item Clear explanations of any assumptions. [Yes]     
 \end{enumerate}

 \item For all figures and tables that present empirical results, check if you include:
 \begin{enumerate}
   \item The code, data, and instructions needed to reproduce the main experimental results (either in the supplemental material or as a URL). [Yes]
   \item All the training details (e.g., data splits, hyperparameters, how they were chosen). [Yes]
    \item A clear definition of the specific measure or statistics and error bars (e.g., with respect to the random seed after running experiments multiple times). [Yes]
    \item A description of the computing infrastructure used. (e.g., type of GPUs, internal cluster, or cloud provider). [Yes]
 \end{enumerate}

 \item If you are using existing assets (e.g., code, data, models) or curating/releasing new assets, check if you include:
 \begin{enumerate}
   \item Citations of the creator If your work uses existing assets. [Not Applicable]
   \item The license information of the assets, if applicable. [Not Applicable]
   \item New assets either in the supplemental material or as a URL, if applicable. [Not Applicable]
   \item Information about consent from data providers/curators. [Not Applicable]
   \item Discussion of sensible content if applicable, e.g., personally identifiable information or offensive content. [Not Applicable]
 \end{enumerate}

 \item If you used crowdsourcing or conducted research with human subjects, check if you include:
 \begin{enumerate}
   \item The full text of instructions given to participants and screenshots. [Not Applicable]
   \item Descriptions of potential participant risks, with links to Institutional Review Board (IRB) approvals if applicable. [Not Applicable]
   \item The estimated hourly wage paid to participants and the total amount spent on participant compensation. [Not Applicable]
 \end{enumerate}

 \end{enumerate}

\appendix

\include{supplement}

\end{document}

%% file: supplement.tex
\onecolumn













\textbf{\begin{center} \LARGE Appendices \end{center}}\hfill \\

\section{Experimental details}\label{app:experimental_details}
\subsection{Dynamics Prediction}
\subsubsection{Experimental setup} Our architecture consists of four GMConv layers with 128 channels each, utilizing PReLU activations and residual connections. We maintain spatial resolution throughout the network by avoiding pooling operations. A final $1 \times 1$ convolution maps the features to the desired output channels. The neighborhood parameters are set to 4 for cyclic groups and 2 for dihedral groups to optimize local interactions. We initialize weights using Kaiming initialization. The model is trained using the AdamW optimizer with learning rates between 0.003 and 0.009, a weight decay of 0.0153, and a ReduceLROnPlateau scheduler (factor 0.7, patience 8) on a mean squared error loss. Training is conducted with a batch size of 256 for up to 100 epochs, employing early stopping (patience 6) based on validation accuracy to prevent overfitting. We utilize PyTorch for our experiments and the hyperparameters are fine-tuned for each dataset.

\subsubsection{Run times} Using an Nvidia L40S GPU (48GB RAM), forward pass times are approximately 1 second for JetFlow ($62 \times 23$ samples) and 1.4 seconds for Plumes ($64 \times 64$ samples), reflecting the complexity of these tasks. FLOPS are detailed in Tables \ref{tab:results_jet} and \ref{tab:results_smoke}.

\subsection{Image Classification}

\subsubsection{Experimental setup} Our base architecture consists of two GMConv layers (120 channels each) with residual connections, LayerNorm without learnable parameters, and PReLU activations. After the GMConv layers, an adaptive max pooling layer is applied, followed directly by a single fully connected layer that performs classification with the cross-entropy loss function. We experiment with GMConv layers with a neighborhood of 3, incorporating error addition, and with GMPool. We also include experiments for different neighborhood sizes. Training is conducted with a batch size of 1024 for a maximum of 100 epochs, with early stopping to prevent overfitting. Other hyperparameters are consistent with those used in our dynamics prediction experiments. For pooling-based experiments, we employ a GM-CNN architecture with 3 GMConv layers (44, 44, and 56 output channels).

\subsubsection{Run times} Forward pass times vary by input size: $\sim0.6$s for 24x24 images (smallNORB), $\sim0.7$s for 28x28 (MNIST variants), and $\sim0.8$s for 32x32 (CIFAR-10, NORB). These are on par for all the competing methods. Note that our method can be sped up significantly with custom GPU implementations for handling structured matrices.

\section{Proofs}\label{app:proofs}
\subsection{Proof of Lemma \ref{lem:subgroupdiag}}
\begin{proof}
Direct by definition \eqref{eq:gdiag}. 
For $h,h'\in H$ and $g\in G$, the entry of $B_h$ at position $(h',g)$ is $\delta(g=h'h)$. 
Since $H$ is closed under the group operation, $g=h'h\Rightarrow g\in H$.
\end{proof}

\subsection{Proof of Proposition \ref{prop:gmalgebra}}\label{app:proofgmalgebra}
\begin{proof}
    The result follows by linearity from the following explicit properties of group diagonals, which are direct to check from the definitions: 
    \begin{equation}\label{eq:gdalgebra}
    B_g^T=B_g^{-1}=B_{g^{-1}},\quad B_{g_1} B_{g_2}=B_{g_1g_2}.
    \end{equation}
    Finally, the case of Kronecker product follows from \eqref{eq:gdprod}.
\end{proof}

\subsection{Proof of Proposition \ref{prop:distancebds}}\label{app:proofdistancebds}

\begin{proof}
    For a matrix $M$ let $M_0$ be $M$'s projection, i.e. the group matrix realizing the distance to $\mathcal{GM}$. Note that $\|M_0\|\leq \|M\|$ is a consequence of Frobenius norm coming from an inner product.
    
    Since for all matrices $X$ we have $\|X\|=\|X^T\|$, and since $\mathcal{GM}$ is closed under transpose, we have $\mathrm{dist}(M, \mathcal{GM})=\|M-M_0\|=\|M^T-M_0^T\|\geq \mathrm{dist}(M^T, \mathcal{GM})$. By interchanging the roles of $M,M^T$ we also find the other inequality, and thus prove item 1.

    For item 2, note that for all matrices $X,Y$ we have $\|XY\|\leq \|X\|\|Y\|$. Then we use the closure property of $\mathcal{GM}$ under product, and triangle inequality, in order to write $\mathrm{dist}(MN, \mathcal{GM})\leq\|MN - M_0N_0\|=\|M(N-N_0) + (M-M_0)N_0\|\leq \|M\|\|N-N_0\|+\|M-M_0\|\|N_0\|\leq\max\{\|M\|, \|N_0\|\}(\|M-M_0\| + \|N-N_0\|)$, which allows to conclude by noting $\|N_0\|\leq \|N\|$.

    For item 3, note that for all $X,Y$ we have $\|X\otimes Y\|\leq \|X\|\|Y\|$. Then we proceed as for item 2, with tensor product replacing matrix product, keeping in mind that, due to Prop. \ref{prop:gmalgebra}, if $M_0\in\mathcal{GM}^G, N_0\in\mathcal{GM}^H$ then $M_0\otimes N_0\in\mathcal{GM}^{G\times H}$.
\end{proof}.

\section{Extended discussion on notions of displacement rank for group matrices}\label{app:ldr} 
\subsection{Comparison to classical displacement dimension for group matrices}\label{app:gaderwaterhouse}
In this section, we review previous notions of displacement rank from \cite{waterhouse1992displacement, gader1990displacement}. These works were motivated by aiming to generalize different aspects of LDR theory from the case of circulant matrices, formulated very similarly to our Example \ref{ex:circulant}. In the case of group matrices, it was realized that the property of circulant matrices of being constant along diagonals can be generalized to the notion of \emph{group diagonals} (whose name includes the term ``diagonals'' for this reason). Then the natural idea is that displacement needed to have the property of vanishing on group matrices as in~\cite{gader1990displacement, waterhouse1992displacement}.

In a formula extending the Stein-type displacement operator $\Delta_{A,B}M = M - AMB$, which we defined in Section \ref{sec:intro}, \cite{gader1990displacement} defined a displacement operator $\mathcal G(M):=M - P\mathcal Q(M)$, in which $P$ is the permutation matrix corresponding to the cyclic permutation of basis vectors $e_1\mapsto e_2\mapsto\cdots\mapsto e_{|G|}\mapsto e_1$, and $Q(M)$ is an involved operation that cyclically permutes elements of $M$ of each given group diagonal pattern separately. The displacement operation $\mathcal G(M)$ from the work of \cite{gader1990displacement} was further rationalized and made more elegant in \cite{waterhouse1992displacement}, which replaced the term $P\mathcal Q(M)$ by a more general form $\mathcal T(M)$ in which the permutation $P$ can be chosen arbitrarily. In both cases, operator $\mathcal G(M)$ subtracts elements of $M$ which are shifted along group diagonal patterns $B_g$, and as the shift is done via a cyclic permutation, we have the property that if all such differences are zero, then the entries of $M$ are constant along group diagonals.

Our contribution is to simplify the expressions further, by introducing the intermediate reordering $F(M)$, which was not present in previous work. This allows to easily operationalize the implementation that is the focus of this work. Our definition \eqref{eq:dr}, and the more complex one \eqref{eq:dr2}, are in direct parallel to the treatment from \cite{gader1990displacement}, \cite{waterhouse1992displacement} respectively. The main justification for using \eqref{eq:dr2}, is that it will be helpful in the proof of Proposition \ref{prop:displdim} below.

\subsection{Displacement dimension changes under elementary operations}\label{app:displ_dim}
In this section, we consider displacement error $\mathsf{D}_{\vec P}(M)$ as defined in \eqref{eq:dr2}, and the derived notion of displacement dimension $\mathrm{dim}_{\mathsf{D}}(\mathcal M)$ from \eqref{eq:dimd}. For the latter, we are interested in how it behaves under elementary operations like those from~\cite{thomas2018learning, pan2012structured}. Note that in previous work on classical displacement rank \cite{thomas2018learning}, similar results are formulated to study the closure properties of classical structured matrices, however, we were not able to adapt displacement rank \eqref{eq:disprank} similarly to our case (as also mentioned in Section \ref{sec:approxequiv}, above proposition \ref{prop:distancebds}) for the following reasons: (1) the results valid for classical LDR structured matrices are not valid in our case, (2) the neural network interpretation as number of degrees of freedom as quantified in the error to precise equivariance is more important to our setup. 

We summarize our results in the following:
\begin{prop}\label{prop:displdim}
    Let $G, H$ be two groups. Further, consider classes of $|G|\times|G|$-matrices $\mathcal M, \mathcal M'$ and a class of $|H|\times |H|$-matrices $\mathcal N$, with $\mathsf{D}$-dimensions with respect to the corresponding groups denoted respectively as $d_M, d_{M'} d_N$. Then the following holds:
    \begin{enumerate}
        \item The class $\mathcal M^T=\{M^T:\ M\in\mathcal M\}$ has $\mathrm{dim}_{\mathsf{D}}(\mathcal M^T)=d_M$.
        \item The set $\mathcal M+\mathcal M':=\{M+M':\ M\in\mathcal M, M'\in \mathcal M'\}$ has $\mathrm{dim}_{\mathsf{D}}(\mathcal M+\mathcal M')=d_M+d_{M'}$.
        \item Consider the set $\mathcal M\otimes \mathcal N:=\{M\otimes N:\ M\in\mathcal M, N\in\mathcal N\}$. Then with respect to the group diagonals from a direct product group $G\times H$, we have $\mathrm{dim}_{\mathsf D}(\mathcal M\otimes\mathcal N)\leq d_M+d_N$.
    \end{enumerate}
\end{prop}
\begin{proof}
    In the setting of item 1, we first claim that \begin{equation}\label{eq:fmt}
    F(M^T)_{gh}=F(M)_{g^{-1},g^{-1}h}.
    \end{equation}
    Once this is proved, we note first that $F(M^T)$'s rows can be re-labeled by $g^{-1}\mapsto g$, to obtain the matrix $\widetilde F(M)$ with entries $F(M)_{g,gh}$. Then, if $\tau_g(h):=gh$, and $\sigma_g$ was the permutation used in the definition of $\mathsf{D}_{\vec P}(M)$ in \eqref{eq:dr2}, then for the rows of $\widetilde F(M)$ we can use cyclic permutation $\tau_g^{-1}\sigma_g\tau_g$. After this change coordinates $\widetilde F(M)\mapsto F(M^T)$. Since all the applied transformations do not depend on the choice of $M\in\mathcal M$, and do not change the dimension counts (since, as observed before, $\mathrm{dim}_{\mathsf{D}}$ does not depend on the choices of $\sigma_g$'s), this means that $\mathcal M, \mathcal M^T$ have the same $\mathsf{D}$-dimension.

    It now suffices to prove \eqref{eq:fmt}. For this, we first note the following commutation relation of vectors with group diagonal matrices $B_g$:
    \begin{equation}\label{eq:fgbg}
        B_g\mathrm{diag}(v)=\mathrm{diag}(w)B_g, \quad \text{where}\quad w_h=v_{g^{-1}h}, h\in G.
    \end{equation}
    The above follows directly using the definition $(B_g)_{h,h'}=\delta(h=gh')$. 
    
    Recall that $B_g$ is a permutation matrix, thus $B_g^T=B_g^{-1}$ and we can directly verify $B_gB_h=B_{gh}$, $B_{id}=Id$, thus 
    \begin{equation}\label{eq:bgt}
        B_g^T=B_g^{-1}=B_{g^{-1}}.
    \end{equation}
    Using \eqref{eq:fgbg}, \eqref{eq:bgt} and a change of variable $g\mapsto g^{-1}$ in the sums, we now get, denoting $F=F(M)$,
    \[
        M^T=\sum_g B_g^T\mathrm{diag}((F_{gh})_h) = \sum_g \mathrm{diag}((F_{g,gh})_h B_g^{-1} = \sum_g \mathrm{diag}((F_{g^{-1},g^{-1}h})_h) B_g.
    \]
    This directly implies \eqref{eq:fmt}, as desired.

    Item 2 follows directly from linearity and using the property that $\mathrm{dim}(V+W)\leq \mathrm{dim}V + \mathrm{dim}W$ for any vector subspaces $V, W\subseteq \mathbb R^{|G|\times|G|}$.

    Item 3 follows from the definition of Kronecker product, observing that $F(M\otimes N) = F(M)\otimes F(N)$.
\end{proof}

\section{Generalization of Framework to Homogeneous Spaces and Infinite Groups}
\label{app:homogneous}

\subsection{Actions on homogeneous spaces.} 

In most machine learning scenarios, the input is more frequently defined not as a function over a group, but rather on a space on which the group acts~\cite{Kondor2018OnTG}. That is, the space $X$ that encodes the inputs, is not identifiable with the symmetry group $G$, but rather is a \emph{homogeneous space of $G$}, i.e. it can be identified with a quotient by a subgroup $H\subseteq G$, denoted $X=G/H$. Elements $x\in X$ are then identified with subsets of the partition into right cosets $[x]:= xH=\{\bar x\in G:\ \exists h\in H, \ \bar x= xh\}$. In general, $G/H$ still has a natural action of $G$, given by $g\cdot[x]:=[gx]$. 

\textbf{Convolution over $G/H$.} The convolution operations can be encoded in this case as follows. Consider a kernel $\phi:G\to\mathbb R$ and a channel $f:G/H\to \mathbb R$, we define
\begin{equation}\label{eq:homconv}
    \phi\star f([x]):=\sum_{g\in G}\phi(g) f(g^{-1}[x])=\sum_{g\in G,[y]\in G/H}\phi(g)f([y])\delta([y]=[g^{-1}x]).
\end{equation}
The above operation can be expressed using group matrices by observing the following
\begin{eqnarray}\label{eq:homconvaux}
    \delta\left([y]=[g^{-1}x]\right) &=& \delta(\exists h\in H:\ yh=g^{-1}x)=\sum_{h\in H}\delta(yh=g^{-1}x)=\sum_{h\in H}\delta(x=gyh)\nonumber\\
    &=&\sum_{h\in H}(B_{gy})_{xh} = \left[B_{gy}\mathbf 1_H\right]_x,
\end{eqnarray}
in which $\mathbf 1_H$ is the characteristic vector of $H$. One can verify that $\left(B_{gy}\mathbf 1_H\right)(x) = \left(B_{gy'}\mathbf 1_H\right)(x')$ whenever $[x]=[x']$ and $[y]=[y']$, thus expression \eqref{eq:homconvaux} does not depend on the choices of representatives $x,y$. Using this observation, we can fix a choice of representatives, namely
\[
    X\subset G,\quad \text{such that}\quad \forall gH\in G/H, \ |gH\cap X|=1. 
\]
Then we can equivalently work with $f, \phi\star f:X\to \mathbb R$. We then write \eqref{eq:homconv} in terms of group diagonal matrices and representatives from $X$ as follows:
\begin{equation}\label{eq:homgmconv}
    \phi\star f(x) = \sum_{g\in G, y\in X} \phi(g) f([y])\left(B_{gy}\mathbf 1_H\right)(x).
\end{equation}

\subsection{Infinite discrete groups and padding.}\label{app:padding} 

Although for ease of exposition in Section \ref{sec:formal} we used periodic translations, more traditionally CNNs use the group $G=\mathbb Z^2$, which is infinite. Then, operation \eqref{eq:gconv} involves an infinite sum, and is thus not computable. This is naturally taken care of by the following adjustments, which we state directly for general groups $G$: 
\begin{enumerate}
    \item We work only with inputs $\psi$ of support contained in a fixed finite set $X_{in}$ (for CNNs, $X_{in}$ is a square $ \{0,\dots,n-1\}^2\subset\mathbb Z^2$).
    \item We use kernels $\phi$ of support constrained to a finite set $\mathcal N$, typically a radius-$k$ (in the word metric induced by a set of generators) neighborhood of the identity. We further assume $\mathcal N_k$ to be symmetric (we use a symmetric set of generators), in which ``symmetric'' means closed under taking inverses: $g\in \mathcal N\Leftrightarrow g^{-1}\in\mathcal N$. 
    \item While applying $\phi\star\psi(x), x\in X_{in}$, the computation \eqref{eq:gconv} involves terms $\psi(g^{-1}x)$, in which sometimes $g^{-1}x\notin X_{in}$. We then have to extend $\phi$ by zero on the set $(X_{in})_{\mathcal N}$, where we use notation \[
        (A)_B:=\{b^{-1}a:\ a\in A, b\in B\}=\mathsf{supp}(1_B\star 1_A).
    \]
    Note that if $\mathcal N$ is the radius-$k$ word-distance ball around the identity, then $(A)_{\mathcal N}$ also can be described as the set of all elements at word-distance $\le k$ from $A$.
\end{enumerate}
The above operations can be summarized as follows:
\begin{eqnarray}\label{eq:padding}
    \mathsf{Pad}(\psi)(x)&:=&\left\{\begin{array}{ll}\psi(x)&\text{ if }x\in X_{in},\\ 0 &\text{ if }x\in \partial_{\mathcal N}X_{in}:=(X_{in})_{\mathcal N}\setminus X_{in},\end{array}\right.\nonumber\\
    \left[\phi\ \star\ \mathsf{Pad}(\psi)\right](x)&:=&\sum_{g\in\mathcal N}\phi(g) \mathsf{Pad}(\psi(g^{-1}x))\quad\text{for}\quad x\in X_{in}.
\end{eqnarray}
Formally, extension by zero corresponds to a subspace immersion given by a matrix multiplication $E:\mathbb R^{X_{in}}\to\mathbb R^{(X_{in})_{\mathcal N}}$ with entries $E_{x,y}=\delta(x=y), x\in X_{in}, y\in (X_{in})_{\mathcal N}$, and then we can define analogues of group matrices $B_g^{(X_{in})_{\mathcal N
}}, g\in\mathcal N$ by the same formula \eqref{eq:gdiag} with the restriction of $h,h'\in(X_{in})_{\mathcal N}$ only. After this, the formula for \eqref{eq:padding} (analogous to \eqref{eq:gmconv} for this case) reads
\begin{equation}\label{eq:gmpadd}
    \widetilde{\mathsf{Conv}}_\phi\overrightarrow\psi:= E^T\left(\sum_{g\in\mathcal N}\phi(g)B_g^{(X_{in})_{\mathcal N
}}\right) E\ \overrightarrow\psi.
\end{equation}

\subsection{Error to equivariance.}\label{app:errortoeq} 

Note that operation $\widetilde{\mathsf{Conv}}_\phi$ is not equivariant under $G$-action, simply because $\overrightarrow \psi$ has entries in $X_{in}$, which in general is not a union of $G$-action orbits, and thus is not invariant under the $G$-action. In other words, the restriction map encoded in $I^T$ in  \eqref{eq:gmpadd} is ``the culprit'' responsible for the loss of equivariance, whereas the term in parenthesis in \eqref{eq:gmpadd} is actually equivariant when applied to elements of the image of $E$ (corresponding to functions over $(X_{in})_{\mathcal N}$ that are zero outside $X_{in}$). A benefit of \eqref{eq:gmpadd} is that it makes it straightforward for us to measure the equivariance error of $\widetilde{\mathsf{Conv}}$, by extending the theory from Section \ref{sec:approxequiv}. Working on the image of the extension map $I$, we define a version of the displacement operator, denoted $\widetilde{\mathsf{D}}$, as follows. Consider $\widetilde M:=EE^TM$ where $M:=\left(\sum_{g\in\mathcal N}\phi(g)B_g^{(X_{in})_{\mathcal N}}\right)$: following \eqref{eq:genmatrix}, we can encode this matrix as
\begin{equation}\label{eq:genmatrixpadd}
    \widetilde M=\sum_{g\in\mathcal N} \mathrm{diag}(\widetilde F_g)B_g^{(X_{in})_{\mathcal N}},\quad\text{where}\quad (\widetilde F_g)_x:=\widetilde M_{x,xg^{-1}}.
\end{equation}
Then we form a matrix $F(\widetilde M)$ with rows $\widetilde F_g, g\in\mathcal N$ and define $\mathsf{D}(\widetilde M)$ as in \eqref{eq:dr}. We can bound the displacement dimension of matrices of the form \eqref{eq:genmatrixpadd} by noting that the $M$ defined earlier has $\mathsf{D}(M)=0$ and by linearity of $\mathsf{D}$, only columns of $\widetilde F$ corresponding to elements $x\in X_{in}$ such that there exists $g\in \mathcal N$ such that $g^{-1}x'\in \partial_{\mathcal N}X_{in}$ can contribute to degrees of freedom of $\mathsf{D}(\widetilde M)$. We thus find that for $\mathcal M:=\{\widetilde M\text{ as in \eqref{eq:genmatrixpadd}}\}$ it holds that:
\begin{equation}\label{eq:ddim}
     \mathrm{dim}_{\mathsf{D}}(\mathcal M)\leq |(X_{in})_{\mathcal N}\setminus X_{in}|\ |\mathcal N|,\qquad \mathsf{DR}(\mathcal M)\leq |(X_{in})_{\mathcal N}\setminus X_{in}|.
\end{equation}
The intuitive explanation for \eqref{eq:ddim} in the case of padding for classical CNNs is that \textbf{the number of padding pixels provides rank control for the error to precise equivariance in each layer}.

\section{Equivariance Error Analysis:}
\label{app:error_analysis}
In \cite{wang2022approximately}, the authors use the equivariance error as a metric to quantify the degree to which a model deviates from perfect symmetry under group transformations. This measure is relevant for approximately equivariant networks, which aim to balance the symmetry constraints and model flexibility. The authors argue that approximately equivariant networks can better capture the imperfect symmetries in real-world scenarios via adding soft equivariance regularization during training. However, equivariance error analysis for GM-CNN suggests that while it does not achieve the optimal balance between data and model equivariance error, it consistently outperforms most other methods, including ConvNet, RPP, Lift as seen in Figure~\ref{fig:equivariance-error}. Our analysis raises the possibility that the balance between model and data equivariance error and its relation to overall performance needs further analysis to accurately quantify. GM-CNN's ability to capture relevant symmetries and offer strong predictive capabilities indicate that certain applications may benefit from a more flexible approach to equivariance. Despite RSteer showing equivariance error close to optimal level, GM-CNN offers a competitive balance and a promising framework for real-world applications where both symmetry and model flexibility are crucial.


\begin{figure}
    \centering    \includegraphics[width=1.0\linewidth]{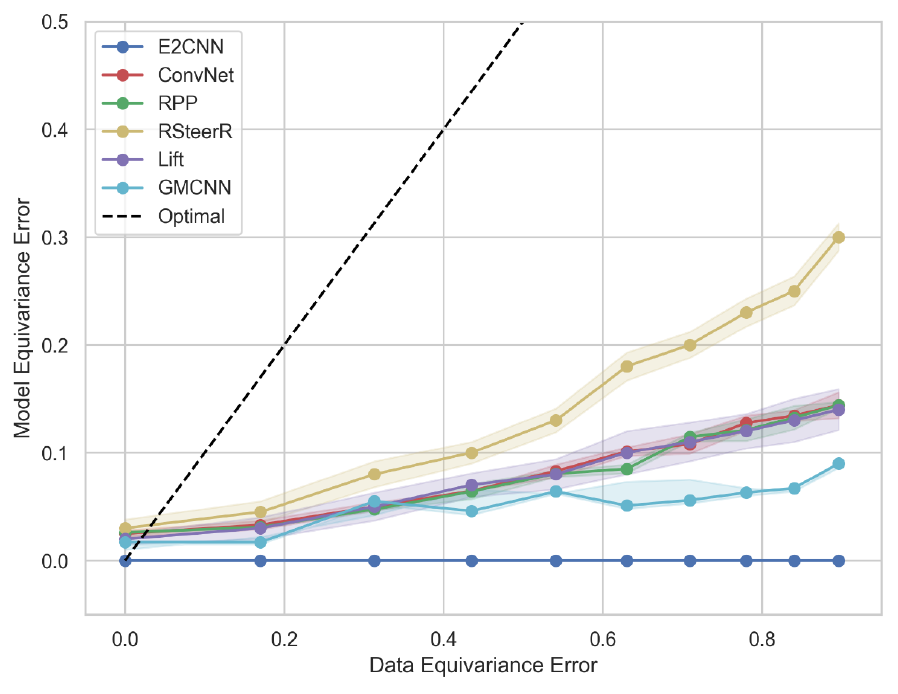}
    \caption{Equivariance error analysis on synthetic smoke plume with different levels of
rotational equivariance as described in \cite{wang2022approximately}.}
    \label{fig:equivariance-error}
\end{figure}
